\documentclass{article}
\usepackage{spconf,amsmath,graphicx}
\usepackage{bm}
\usepackage{xcolor}
\usepackage{url}
\usepackage{algorithm}
\usepackage[]{algpseudocode}
\usepackage{amsmath}
\usepackage{multirow}
\usepackage{graphicx}
\usepackage{subcaption}
\usepackage{amsfonts}
\usepackage{enumitem}
\usepackage{amssymb, amsthm}
\usepackage{algorithm}
\usepackage{algpseudocode}
\usepackage{xcolor}
\usepackage{graphicx}
\usepackage{amsthm}
\usepackage{pdfpages}

\allowdisplaybreaks

\newcommand{\BB}[1]{\mathbb{#1}}
\newcommand{\BF}[1]{\mathbf{#1}}
\newcommand{\BC}[1]{\mathcal{#1}}
\newcommand{\proj}[2]{\pi_{#1}(#2)}
\newcommand{\das}[1]{\textcolor{black}{#1}}

\newcommand{\highlight}[1]{\textcolor{black}{#1}}
\newcommand{\Mod}[1]{\ensuremath{\ (\mathrm{mod}\ #1)}}
\newtheorem{theorem}{Theorem}[section]
\newtheorem{lemma}{Lemma}
\newtheorem{assumption}{Assumption}
\newcommand{\Lc}{{\mathcal{L}}}
\newcommand{\pdj}{{\nabla_{(j)}}}
\newcommand{\gm}{\tilde{\bm{\theta}}}
\newcommand{\gmj}{\tilde{\bm{\theta}}_{(j)}}
\newcommand{\tjk}{\bm{\theta}_{k,j}}
\newcommand{\ykj}{\bm{y}_{k,j}}
\newcommand{\tjminus}{\bm{\theta}_{-j}^{t_0}}
\newcommand{\gkj}{ g_{k,j}}
\newcommand{\zb}{\zeta}

\newcommand{\ex}{\mathbb{E}}
\newcommand{\Gj}{\bm{G}_{(j)}}
\newcommand{\Hj}{\bm{H}_{(j)}}
\newcommand{\squeezeuppicture}{\vspace{-3mm}} %reduce whitespace after pictures
\newcommand*{\QEDB}{\null\nobreak\hfill\ensuremath{\square}}%

\newcommand{\eqdef}{\overset{\mathrm{def}}{=}}
\newcommand{\ext}{\ex_t}

\title{MULTI-TIER FEDERATED LEARNING FOR VERTICALLY PARTITIONED DATA}
\name{Anirban Das, Stacy Patterson
\thanks{This work is supported by the Rensselaer-IBM AI Research Collaboration (http://airc.rpi.edu), part of the IBM AI Horizons Network (http://ibm.biz/AIHorizons), and by the National Science Foundation under grants CNS 1553340 and CNS 1816307.}
}
\address{ Department of Computer Science\\
Rensselaer Polytechnic Institute, Troy, New York, USA}

\begin{document}
\ninept
\maketitle
\begin{abstract}
We consider decentralized model training in tiered communication networks.
Our network model consists of a set of silos, each holding a vertical partition of the data. Each silo contains a hub and a set of clients, with the silo's vertical data shard partitioned horizontally across its clients.
We propose Tiered Decentralized Coordinate Descent (TDCD), a communication-efficient decentralized training algorithm for such two-tiered networks. 
To reduce communication overhead, the clients in each silo perform multiple local gradient steps before sharing updates with their hub. 
Each hub adjusts its coordinates by averaging its workers' updates, and then hubs exchange intermediate updates with one another. 
We  present a theoretical analysis of our algorithm and show the dependence of the convergence rate on the number of vertical partitions, the number of local updates, and the number of clients in each hub. 
We further validate our approach empirically via simulation-based experiments using a variety of datasets and both convex and non-convex objectives.
\end{abstract}
\begin{keywords}
vertical machine learning, coordinate descent, federated learning, stochastic gradient descent
\end{keywords}
%

%%%%%%%%%%%%%%%%%%%%%%%%%%%%%%%%%%%%%%%%%%%%%%%%%%%%%%%%%%%%%%%%%%%%%%%%%%%%%%%%
%%%%%%%%%%%%%%%%%%%%%%%%%%  INTRODUCTION  %%%%%%%%%%%%%%%%%%%%%%%%%%%%%%%%%%%%%%
%%%%%%%%%%%%%%%%%%%%%%%%%%%%%%%%%%%%%%%%%%%%%%%%%%%%%%%%%%%%%%%%%%%%%%%%%%%%%%%%
\section{Introduction}
\label{sec:intro}

In recent times, we have seen an exponential increase of data produced at the edge of the communication networks. In many settings, it is infeasible to transfer the entire dataset to a centralized cloud for downstream analysis, either due to practical constraints such as high communication cost or latency, or to maintain user privacy and security~\cite{kairouz2019advances}. This has led to the deployment of distributed machine learning and deep-learning techniques where computation is performed collaboratively by set of clients, each close to its own data source.

Once scenario that arises in distributed training is 
when clients have different sets of features, but there is a sizable overlap in the sample ID space among their datasets~\cite{webank_reviewPaper}. 
For example, the  training dataset may be distributed across silos in a multi-organizational context, for example in healthcare, banking, finance, retail, etc.~\cite{webank_reviewPaper, sun2019privacy}. 
Each silo holds a distinct set of features (e.g., customer/patient list);
the data within each silo may even be of a different modality, for example, one silo may have audio features, whereas another silo has  image data. 
The paradigm of training a global model over such feature-partitioned data is called \emph{vertical federated learning}~\cite{yang2019parallel, liu2019communication}. 
This is different from the more prevalent alternative of \emph{horizontal learning}, where the participating clients each have the entire set of features for a subset of the sample space~\cite{mcmahan2016communication, konevcny2016federated, kairouz2019advances}.

Earlier vertical learning works~\cite{hardy2017private,yang2019parallel,feng2020multi,chen2020vafl} considered a case where each party needs to communicate in each iteration, which may be expensive communication-wise.
To save communication, multiple rounds of training can be performed on a client before reconciling the local model updates into the global model. A more recent work~\cite{liu2019communication} proposed an algorithm that addresses this problem by performing multiple local training iterations before reconciling the client model updates into the global model.
All of these works assume that the entire dataset of a silo is contained in a single client.
However, this model fails to capture the case where the dataset within a silo is horizontally partitioned across multiple clients,
for example, the dataset of a bank may be distributed among its branches, or healthcare data among hospitals in a chain. 

We propose a training algorithm, tiered decentralized coordinate descent (TDCD), for vertical federated learning where there are multiple clients in each silo.
We consider a two tiered network architecture consisting of multiple silos. Each silo holds a vertical partitioning of the data, and internally consists of a hub and multiple clients connected to the hub. 
The data in a silo is further horizontally distributed among its clients. 
Our goal, is to jointly train a model on the features of the data contained across silos, without explicitly sharing raw data from clients, and only via passing intermediate information vectors. \das{TDCD  works by performing a non-trivial combination of parallel coordinate descent on the top tier between silos, and distributed stochastic gradient descent in the bottom tier of clients inside each silo.
To reduce communication, each client performs multiple local gradient steps before sending updates to its hub.} 
This optimization is similar to the method studied in~\cite{mcmahan2016communication, stich2018local,tlifedprox} for horizontal learning. 
We note that some existing works have proposed training algorithms for hierarchical network architectures~\cite{wang2018cooperative, 9054634, castiglia2020multi, 9148862}, but only from the perspective of horizontal learning. Our approach is thus a novel combination of learning with both vertically and horizontally partitioned data in a multi-tiered network.

Specifically, our contributions are the following: 
(1) we present a system model for decentralized learning in a two-tier network, where data is both vertically and horizontally partitioned;
(2) we develop a communication-efficient decentralized learning algorithm, using principles from coordinated descent and stochastic gradient descent;
(3) we analyze the convergence of our proposed algorithm and show how it depends on the number of silos, the number of clients, and the number of local training rounds;
(4) we validate our analysis via experiments using convex and non-convex objectives.

%%%%%%%%%%%%%%%%%%%%%%%%%%%%%%%%%%%%%%%%%%%%%%%%%%%%%%%%%%%%%%%%%%%%%%%%%%%%%%%%
%%%%%%%%%%%%%%%%%%%%%%%%%%  SYSTEM MODEL  %%%%%%%%%%%%%%%%%%%%%%%%%%%%%%%%%%%%%%
%%%%%%%%%%%%%%%%%%%%%%%%%%%%%%%%%%%%%%%%%%%%%%%%%%%%%%%%%%%%%%%%%%%%%%%%%%%%%%%%
\section{SYSTEM MODEL AND PROBLEM FORMULATION}
\label{sec.system_model}

In this section, we  describe the system architecture, the allocation of the training data, and the loss function we seek to minimize.

\subsection{System Architecture and Training Data}
We consider a decentralized system consisting of $N$ silos, shown Fig.~\ref{fig:system_model}.
Each silo consists of a hub and multiple clients connected to it in a hub-and-spoke fashion. \das{The hub network forms a complete graph. For simplicity, we assume that each silo has $K$ clients.
Our network model thus has two tiers, the top tier of hubs, shown in orange, that communicate with each other,
and the bottom tier of clients in each silo, shown in gray.}

The training data consists of $M$ samples that are common across all silos. Each sample has $D$ features. 
The data is partitioned vertically across the $N$ silos so that each silo owns a disjoint set of  $D_j$ features for all of the $M$ samples.
We can express the entire training dataset by a matrix $\BF{X} \in \BB{R}^{M\times D}$.
We denote set of data, i.e., the columns of $\BF{X}$, held in silo $j$ by $\BF{X}_{(j)}$. 
Within each silo, its data is partitioned horizontally across its clients, so that each client holds some rows of $\BF{X}_{(j)}$.
We denote the horizontal shard of $\BF{X}_{(j)}$ that is held by client $k$ in silo $j$ as $\BF{X}_{k,j}$. 
Lastly, we denote a sample  $i$ of the dataset (single row of $\BF{X}$) as $\BF{X}^{(i)}$, and $\BF{X}^{(i)}_{(j)}$ denotes the features of the $i$th sample corresponding to silo $j$.
We assume that each client stores the sample labels $\BF{y}_{k,j}$ for its data $\BF{X}_{k,j}$.

\begin{figure}[htpb]
\centering
    \includegraphics[width=0.6\linewidth]{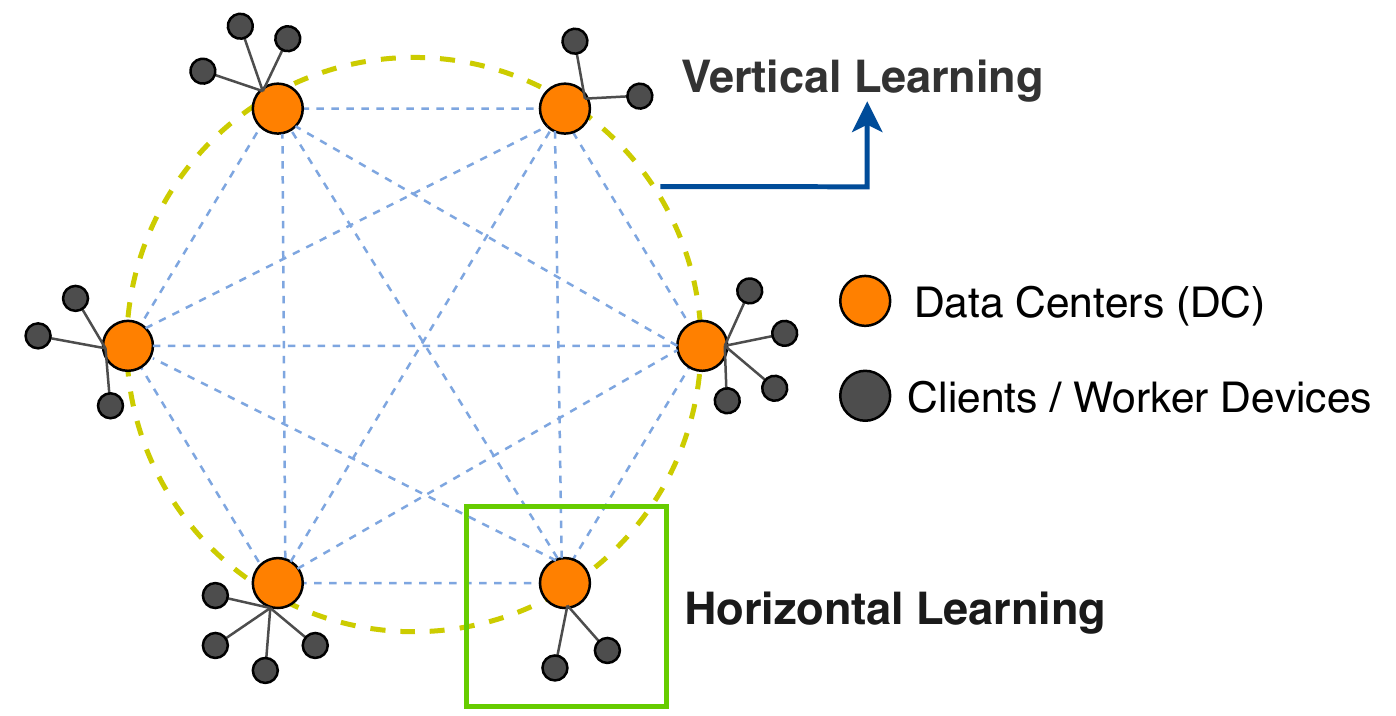}
    \caption{System architecture.} 
    \label{fig:system_model}
\squeezeuppicture
\end{figure}

\subsection{Loss Function}

The objective is to train a global model $\gm$, which is a $d$-vector that can be decomposed as
\[
\gm = [ \gm_{(1)}^T, \ldots, \gm_{(N)}^T]^T
\]
where each $\gmj$ is the block of features, or coordinates, for silo $j$.
 The goal of the training algorithm is to minimize an objective function with following structure:
\begin{align*}
\Lc (\gm,\BF{\BF{X;y}}) \eqdef  \frac{1}{M} \sum\limits_{i=1}^M f(\gm_{(1)}, \ldots, \gm_{(N)}; \BF{X}^{(i)},\BF{y}^{(i)} ) + \lambda \sum_{d=1}^N \omega(\gm_{(d)})
\end{align*}
where $f$ has the partially separable form
\begin{align*}
f(\gm, \BF{X}^{(i)}; \BF{y}^{(i)}) = f \left( \sum\limits_{d=1}^N \BF{X}^{(i)}_{(d)} \gm_{(d)}\;, \BF{y}^{(i)}\right).
\end{align*}
The functions $\omega(\cdot)$ constitute a regularizer, and $\lambda$ is a hyperparameter.  
A concrete example of the loss function is an $L_2$ regularized square loss function for empirical risk minimization:
\begin{align*}
\Lc (\gm, \BF{\BF{X;y}}) &= \frac{1}{2M}||\BF{X}\gm -\BF{y}\|_2^2 + \frac{\| \gm \|_2^2}{2}. 
\end{align*}

%%%%%%%%%%%%%%%%%%%%%%%%%%%%%%%%%%%%%%%%%%%%%%%%%%%%%%%%%%%%%%%%%%%%%%%%%%%%%%%%
%%%%%%%%%%%%%%%%%%%%%%%%%%  ALGORITHM     %%%%%%%%%%%%%%%%%%%%%%%%%%%%%%%%%%%%%%
%%%%%%%%%%%%%%%%%%%%%%%%%%%%%%%%%%%%%%%%%%%%%%%%%%%%%%%%%%%%%%%%%%%%%%%%%%%%%%%%
\section{Proposed Algorithm}
\label{sec.proposedalgorithm}
In this section, we present our Tiered Decentralized Coordinate Descent algorithm (TDCD). 
The pseudocode is given in Algorithm~\ref{alg:algorithm}. 
We first note that the hubs update their own corresponding blocks of coordinates of $\gmj$ in parallel; no hub has the entire $\gm$.  
We define $\tjk^t \in \BB{R}^{D_j}$ as the local version of the coordinates of the weight vector $\gmj^t$ that each client updates. These local versions are initialized by the clients at iteration $t=0$. 

In iteration 0, and every $Q$th iteration thereafter, 
 %At each $Q$th iteration, 
the hubs first average the models from the clients, where hub $j$, updates the $j$th block coordinates of global weight $\gm^t$ as $\gmj^t=\frac{1}{K} \sum\limits_{k=1}^K \left [\tjk^{t}\right]$.
This step is similar to horizontal federated learning. The hubs then agree on $Q$ minibatches $\{ \zeta^{\tau} \}_{\tau=t}^{t+Q-1}$, each containing $B$ samples randomly drawn from the global dataset $\BF{X}$. 
The hubs communicate the aggregated model and the minibatch information to their clients. The clients, in turn, reply with the intermediate information for the samples IDs in those Q minibatches using the newest aggregated model. 
It is necessary to propagate this intermediate information to allow clients in other hubs to calculate partial derivatives during training.  
We define the \das{\textit{intermediate information}} for the $j$th coordinate block for a single sample $p$ as $\Phi_{(j)}^{(p)} =  \BF{X}_{(j)}^{(p)} \gmj^t$. 
For a single minibatch $\zeta$, each client computes a set of information $\Phi_{k,j}^{\zeta} = \{ \Phi_{(j)}^{(p)}\}_{p \in \zeta}$. 
Each client then sends $Q$ such sets of intermediate information to its hub corresponding to the $Q$ minibatches.  
The hub then stacks the set of updates $\{\Phi_{k,j}^{\zeta}\}$ from each of its clients to form $\Phi_t^j$. 
Each hub $j$ then broadcast $\Phi_t^j$ to other hubs to propagate this information. 
For hub $j$, we denote the intermediate information obtained from other hubs by $\Phi_{-j} = \sum_{l=1, l\neq j}^N \Phi_{j}$.
Once this is done, the hub then applies a projection function for each client $k$ to send the subset of information from $\Phi_{-j}$ relevant to client $k's$ samples to that client. 
Alternatively a hub can send the entire $\Phi_{-j}$ to the client and the client can do the projection itself to extract the rows corresponding to its own samples.
We define a projection function $\pi_{k,j}$ such that $\proj{k,j}{\Phi_{-j}} = \Phi_{-k,j}$, where $\Phi_{-k,j}$ is the extracted relevant information for client $k$ of silo $j$.
 
After receiving this intermediate information, at each iteration $t$ each client $k$ of silo $j$ can now calculate its own local partial derivatives of $\Lc$ with respect to coordinate block $j$.
This is denoted by $\gkj$ and is a function of $\Phi_{-k,j}$, the part of $\BF{X}_{k,j}$ in minibatch $\zeta^{t}$, and the local set of weights $\tjk$.
Each client executes $Q$ local stochastic gradient steps, on the features for their respective silos, using a different minibatch in each iteration:
\begin{align}
\tjk^{t+1}= \tjk^{t} - \eta g_{k,j}(\Phi_{-k,j}^{t_0},\tjk^t;\zeta^{\tau}).
\end{align}
$\eta$ is the step size (learning rate), and 
$t_0$ represents the most recent iteration $t_0 < t$ in which the client received intermediate information from its hub.
The entire process is repeated until convergence. 

Informally, each silo effectively takes an approximate (stochastic) gradient step towards the minimizer of $\Lc (\gm^t)$ along the direction of the its coordinates every $Q$ iterations.

\begin{algorithm}[t]
	\caption{Tiered Decentralized Coordinate Descent (TDCD)}
	\begin{algorithmic}[1]
		\State Initialize $\tjk^t= \tjk^{init} \in \BB{R}^{D_j}~,\forall k,j$ 
		\For {$t = 0, \ldots, \infty$}
			\If{t\Mod{Q}=0}            
            	\For {$j = 1, \ldots, N$ silos \emph{in parallel}} \label{line:decen_start}
					\State Hub $j$ computes $\gmj^t = \frac{1}{K} \sum_{k=1}^K \tjk^t$\label{line:federated_avg}%  \\
                	\State Randomly sample $Q$ minibatches $\{\zeta^{\tau}\}_{\tau=t+1}^{t+Q}$ \label{line.pickminibatch}            	
            		\For {$k = 1,\ldots, K$ clients \emph{in parallel}}
                    	\State Set $\tjk^{t} = \gmj^{t}$ \label{line:fedavg_wts_to_client}
                    	\State Send $\Phi_{k,j}^{\zeta}$ to hub $j$, for each $\zeta \in \{\zeta^{\tau}\}_{\tau=t+1}^{t+Q}$\label{line:updates_from_clients_client}
                	\EndFor
                	\State Hub $j$ stack $\{\Phi_{k,j}^{\zeta}\}, \zeta \in \{\zeta^{\tau}\}_{\tau=t+1}^{T+Q}$ to form $\Phi_{j}^{t}$ \label{line:stack_intermediate}
                	\State All hubs exchange $\Phi_{j}^{t}, \forall j=1,\ldots, N$% and $\Theta_{j}^{t}$.
                	\State Hub $j$ calculate $\Phi_{-j}^{t} = \sum \Phi_{p}^{t}, \forall p \neq j$
					\State In parallel set $\Phi_{-k,j}^{t_0} = \proj{k,j}{\Phi_{-j}^t}$ in $K$ clients. \label{line.projection}            
                \EndFor \label{line:decen_stop}
            \EndIf
            
            \For {$j= 1,\ldots, N$ silos \emph{in parallel}} \label{line:dc_train_start}
                \For {$k = 1, \ldots, K$ clients \emph{in parallel}} \label{line:GD_start}
            		\State $\tjk^{t+1}= \tjk^{t} - \eta g_{k,j}(\Phi_{-k,j}^{t_0},\tjk^t;\zeta^{\tau})$ \label{line:local_gradient_update}
                \EndFor \label{line:GD_stop}
			\EndFor     
     \EndFor \label{line:dc_train_stop}                
     \end{algorithmic} 
	\label{alg:algorithm}
\end{algorithm}

In TDCD, clients only communicate their local model and intermediate information every $Q$ iterations. This is in contrast to distributed SGD algorithms, where the clients need to sync with a coordinating hub in each iteration. This allows TDCD to save bandwidth by increasing $Q$, especially when the the size of the model is large.  Hubs still need to exchange intermediate information for all $Q$ minibatches, in between local training rounds. However,  sending all information at the beginning of $Q$ iterations, rather than in every iteration,  potentially saves network latency and overhead. 
The significant bandwidth savings comes in the silos themselves, since each hub and its clients only share the models every $Q$ iterations.
As a rough estimate, the intermediate information for a sample ranges from a simple scalar value to a small vector of very few dimensions. Therefore while training models in deep learning, the intermediate information of B minibatches with M samples each would be of the order of a few megabytes or less. Compared to this, the size of the actual model can be in the order of gigabytes.
We explore how $Q$ impacts the convergence of TDCD in the next section.

\das{We note that at any step of training hubs can communicate their slice of the global model with each other to form the entire global model for use in inference purposes.} 

%%%%%%%%%%%%%%%%%%%%%%%%%%%%%%%%%%%%%%%%%%%%%%%%%%%%%%%%%%%%%%%%%%%%%%%%%%%%%%%%
%%%%%%%%%%%%%%%%%%%%%%%%%%%%%%%%%%%%%%%%%%%%%%%%%%%%%%%%%%%%%%%%%%%%%%%%%%%%%%%%
\begin{figure*}[t]
\centering
\begin{subfigure}[b]{.3\linewidth}
      \centering
    \includegraphics[width=0.88\linewidth]{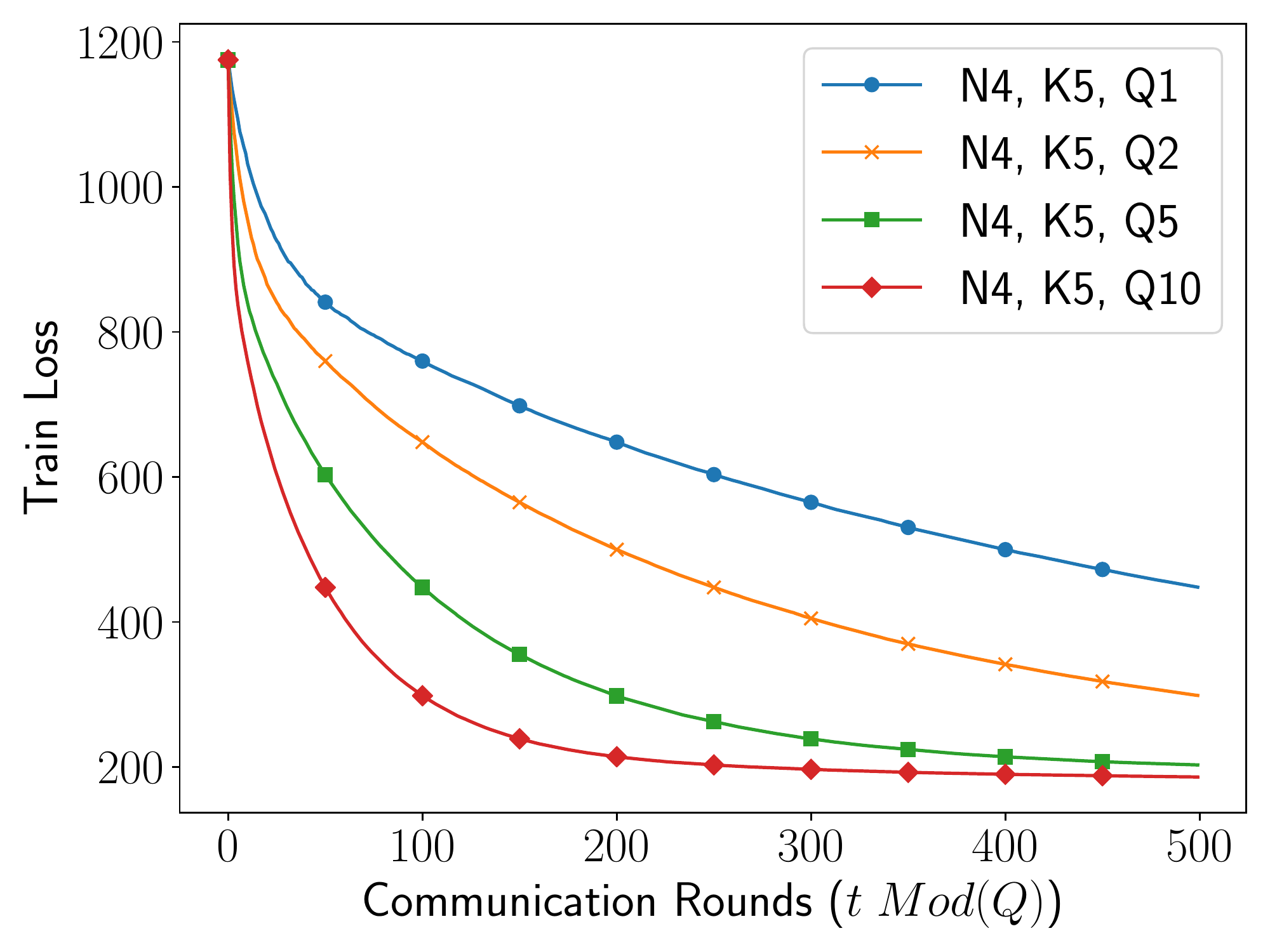}
    \caption{Variation with $Q$}
    \label{fig:ridge_Q_diverge}
\end{subfigure}%
\begin{subfigure}[b]{.3\linewidth}
      \centering
    \includegraphics[width=0.88\linewidth]{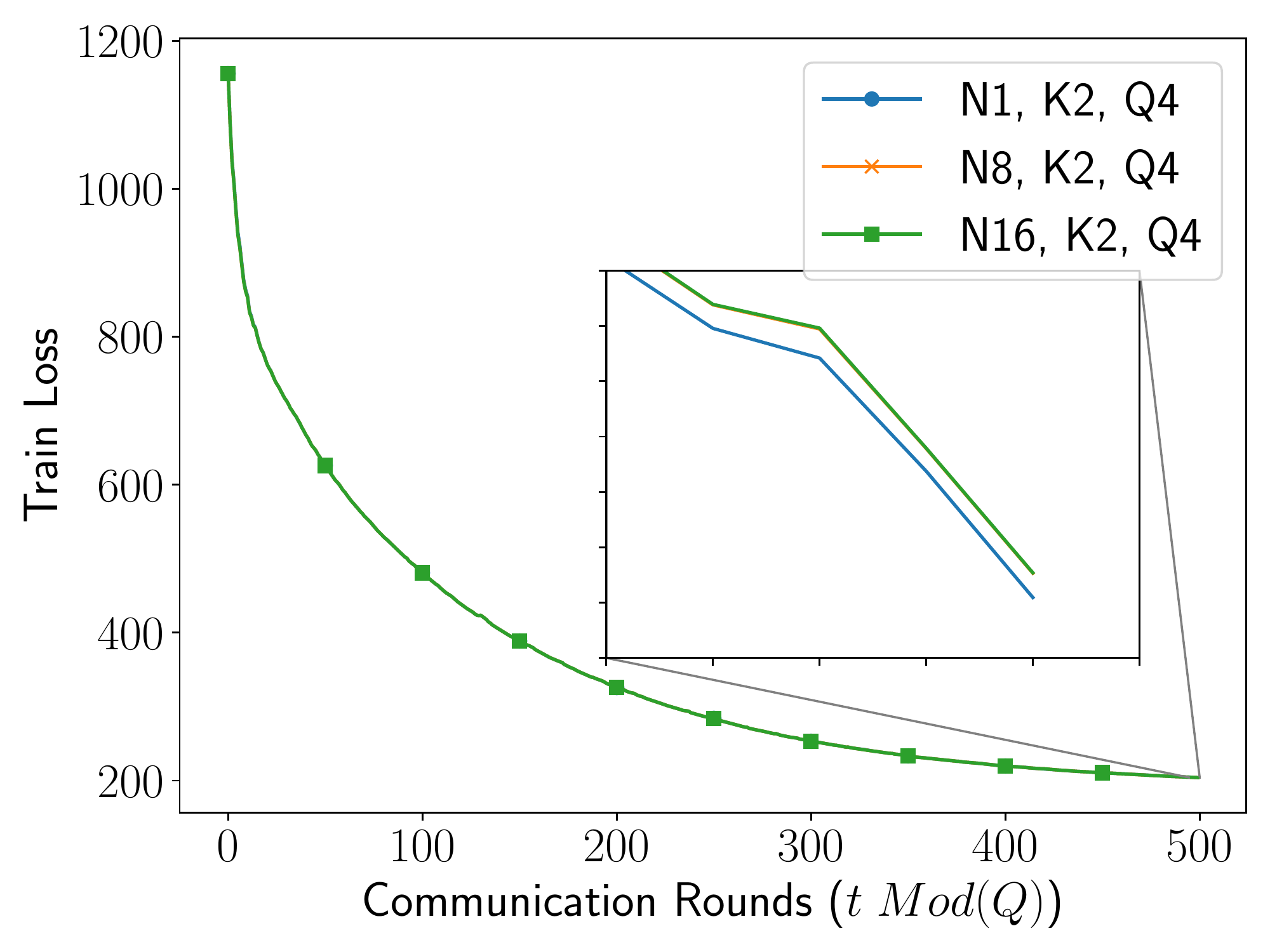}
    \caption{Variation with $N$}
    \label{fig:ridge_N}
     \end{subfigure}
\begin{subfigure}[b]{.3\linewidth}
      \centering%
    \includegraphics[width=0.88\linewidth]{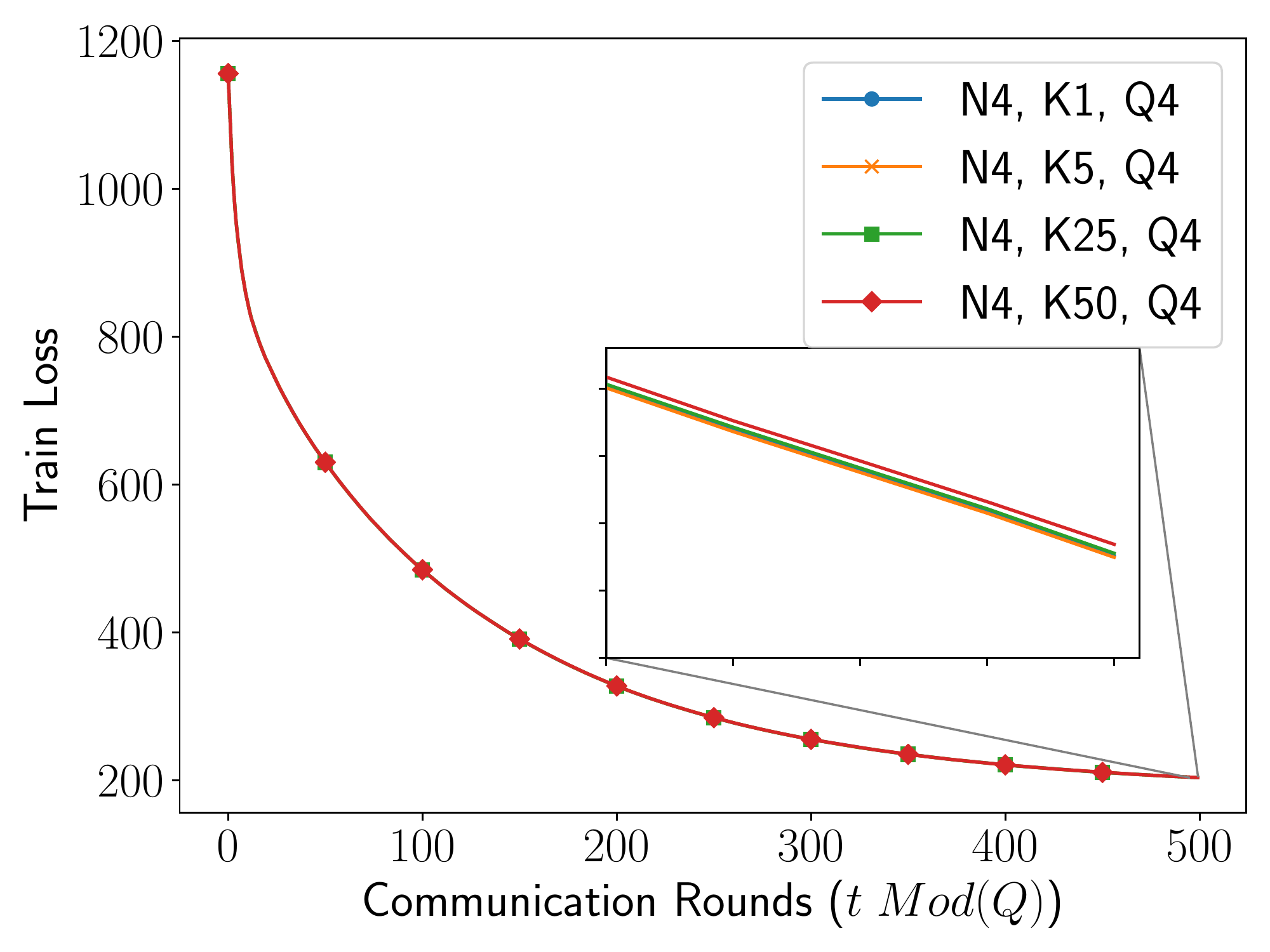}
    \caption{Variation with $K$}
    \label{fig:ridge_K}
     \end{subfigure}
\caption{Ridge Regression Convex Objective. Training loss vs communication rounds for variations of $Q$, $N$ and $K$.}
\label{fig.convex}
\squeezeuppicture
\end{figure*}

\begin{figure}[t]
\centering
\begin{subfigure}[b]{.48\linewidth}
      \centering
    \includegraphics[width=\linewidth]{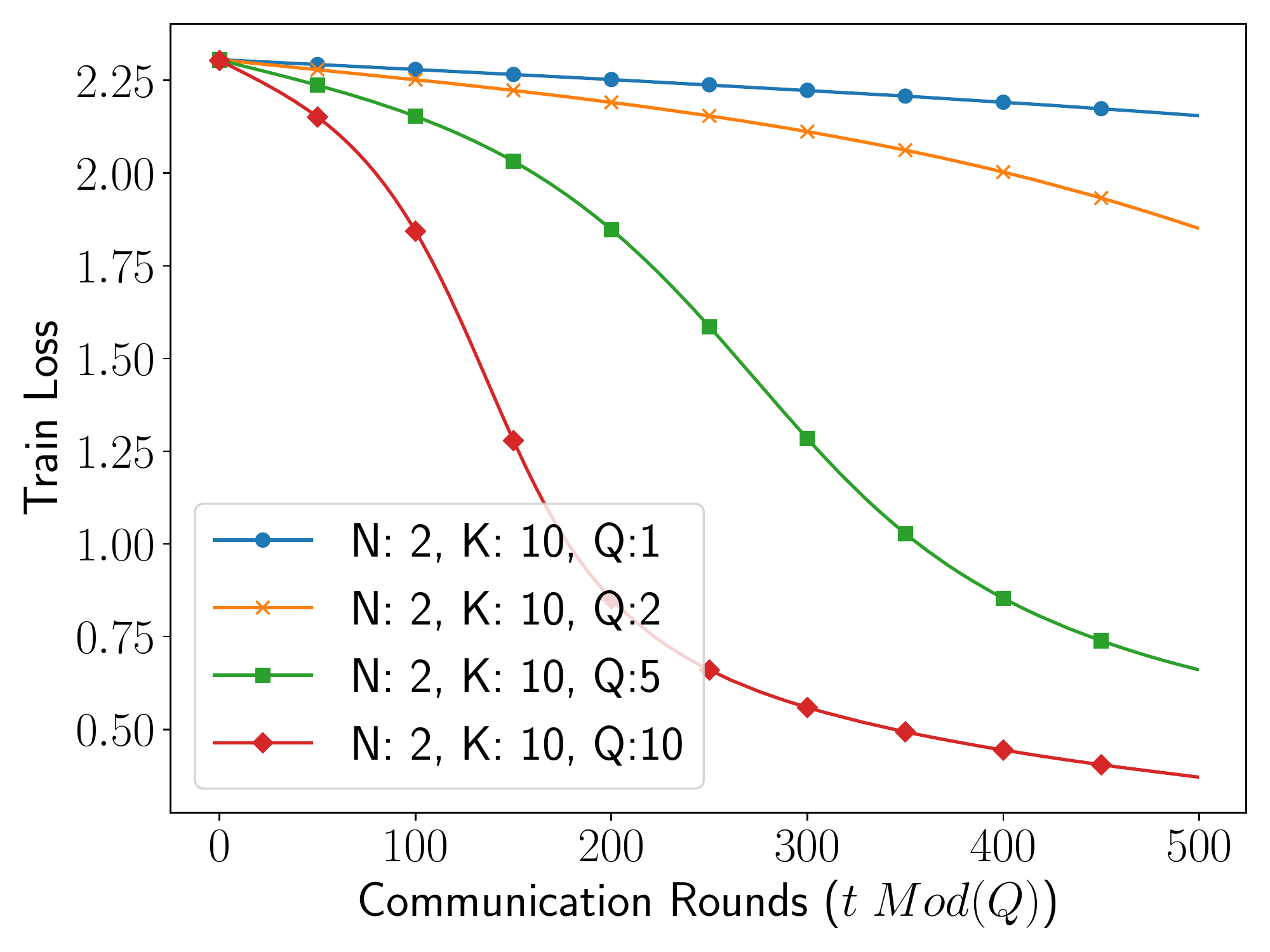}
    \caption{Variation with $Q$}
    \label{fig:comm_vs_Q}
     \end{subfigure}
\begin{subfigure}[b]{.48\linewidth}
      \centering
    \includegraphics[width=\linewidth]{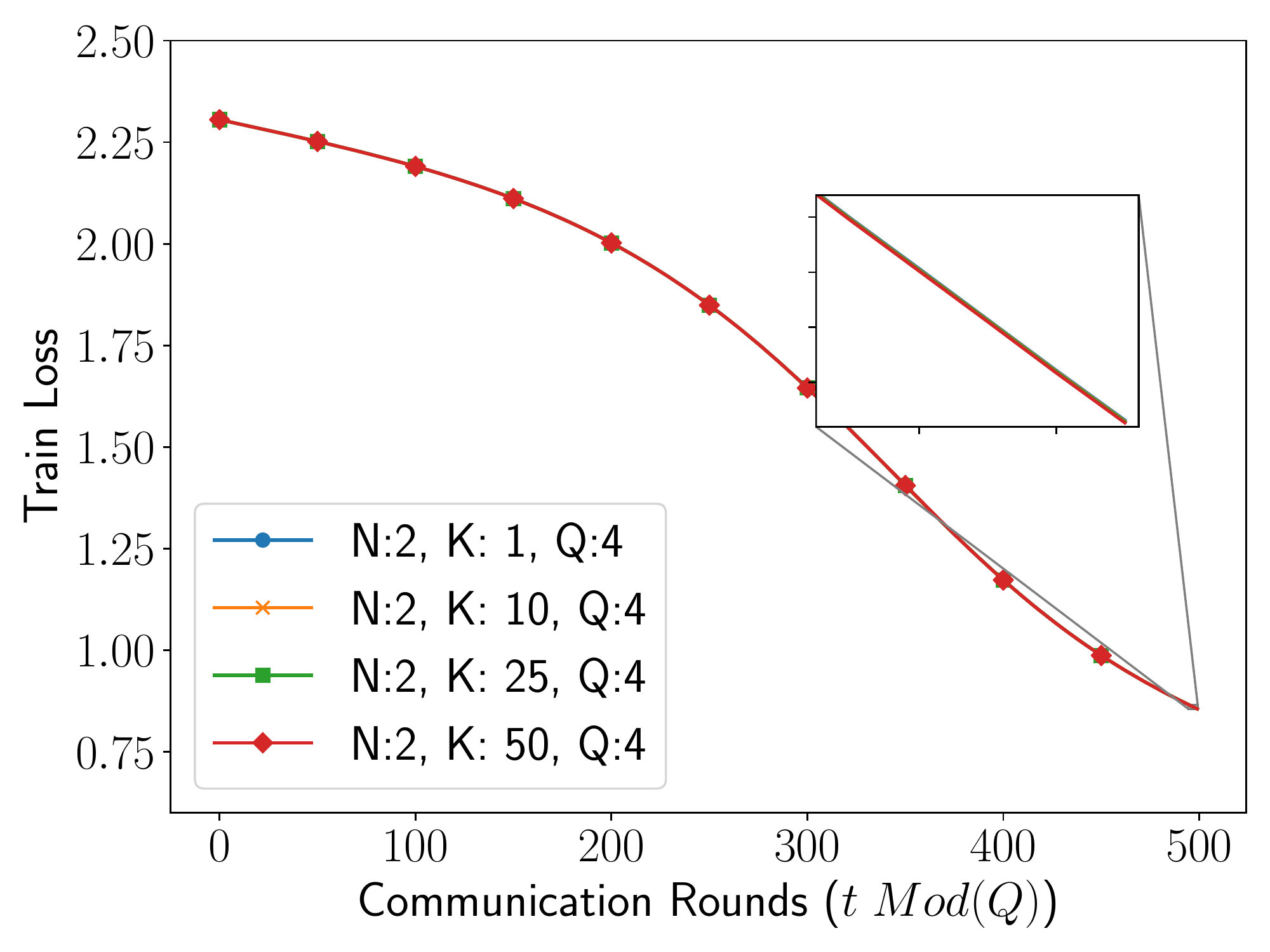}
    \caption{Variation with $K$}
    \label{fig:mnist_k_bsconst}
     \end{subfigure}
\caption{CNN Multi-class classification with Non-Convex Objective. Training loss vs communication rounds for variations of $Q$ and $K$.}
     \label{fig.non-convex}
\squeezeuppicture
\end{figure}
%%%%%%%%%%%%%%%%%%%%%%%%%%%%%%%%%%%%%%%%%%%
%%%%%%%%%%%%%%%%%%%%%%%%%%%%%%%%%%%%%%%%%%%%%%%%%%%%%%%%%%%%%%%%%%%%%%%%%%%%%%%%
%%%%%%%%%%%%%%%%%%%%%%%%%%  Conv. Analysis%%%%%%%%%%%%%%%%%%%%%%%%%%%%%%%%%%%%%%
%%%%%%%%%%%%%%%%%%%%%%%%%%%%%%%%%%%%%%%%%%%%%%%%%%%%%%%%%%%%%%%%%%%%%%%%%%%%%%%%
\section{Convergence Analysis}
\label{sec.convanalysis}
In this section, we provide the convergence analysis of the TDCD algorithm. 
Our analysis is based on the evolution of the global model $\gm \in \BB{R}^D$ following Algorithm \ref{alg:algorithm}. It to be noted that the components of $\gm$,  $\gmj$ are realized every $Q$ iterations, but we will study the evolution of a virtual  $\gmj$ at each iteration,  $\gmj^t= \frac{1}{K} \sum_{k=1}^K \tjk^t$.

To facilitate the analysis, we first define the notion of an auxiliary local vector, which represents the local view of the global model at each client. 
Let $y_{k,j}^t$ denote the auxiliary weight vector used by client $k$ in hub $j$ to calculate the partial derivative $g_{k,j}(\ykj^t)$, 
\begin{align}
\ykj^t = [\tjminus, \tjk^t] 
\end{align}
where, $\tjminus$ denotes the vector of all coordinates of $\gm$ excluding block $j$ at iteration $t$, where $t_0$ is the iteration when the client $k$ last updated the value of $\tjminus$ from its hub. 
Therefore, when a client takes multiple local steps to update $\tjk$, it uses a stale value of the elements in the other coordinates of $\ykj$. 

We further define the following two quantities:
\begin{align}
&\bm{G}^t = [(\bm{G}_{(1)}^t)^T, \ldots, (\bm{G}_{(1)}^t)^T]^T \text{ , } \Gj^t =  \frac{1}{K} \sum_{k=1}^K \gkj(\ykj)
%\\
%&\bm{H}^t = [(\bm{H}_{(1)}^t)^T, \ldots, (\bm{H}_{(1)}^t)^T]^T \text{ , } \Hj^t = \frac{1}{K} \sum_{k=1}^K \pdj \Lc (\ykj)
\end{align}
We can then write the evolution of the global model as follows, 
\begin{align}
\gm^{t+1} = \gm^t - \eta \bm{G^t} \label{eqn.original}
\end{align}

We make the following assumptions about the loss function $\Lc$ and the gradients $\gkj$ at each client.
\begin{assumption} \label{assumption.lipschitz} 
    The gradient of the loss function is Lipschitz continuous with constant $L$; further, the partial derivative of $\BC{L}$ with respect to each coordinate block $j$ is Lipschitz continuous with constant $L_j$, i.e.,  for all $\bm{\theta}_1, \bm{\theta}_2 \in \mathbb{R}^D$   
   \begin{align}
     &\lVert \nabla \Lc(\bm{\theta}_1) - \nabla \Lc(\bm{\theta}_2) \parallel \leq L\parallel \bm{\theta}_1 - \bm{\theta}_2\rVert \\         
    &\lVert \pdj \Lc(\bm{\theta}_1) - \pdj \Lc(\bm{\theta}_2) \parallel \leq L_j\parallel \bm{\theta}_1 - \bm{\theta}_2\rVert.     
    \end{align}
   \end{assumption}

\begin{assumption}\label{assumption.lowerbound}
The function $\Lc$ is lower bounded so that for all $\bm{\theta} \in \mathbb{R}^D, \Lc(\bm{\theta}) \geq \Lc_{inf}$.
\end{assumption}

\begin{assumption}\label{assumption.unbiased}
Let $\zeta$ be a mini-batch drawn uniformly at random from all samples. 
We assume that the data is distributed so that, for all $\bm{\theta} \in \mathbb{R}^D$
\begin{align}
    &\ex_{\zeta | \bm{\theta}}\left[\gkj(\bm{\theta}) \right]= \pdj \Lc(\bm{\theta})\\
      &\ex_{\zeta | \bm{\theta}} \left[\| \gkj(\bm{\theta}) - \pdj \Lc(\bm{\theta}) \|^2 \right] \leq \sigma^2_j.
\end{align}
\end{assumption}
We also use the following definitions: 
\[
L_{max} = \max_{1\leq j \leq N} L_j ~,~ \sigma_{max} = \max_{1 \leq j \leq N} \sigma_j
\] 

We now provide the main theoretical result of the paper. The proof is deferred to a technical report available in the Appendix~\ref{appendix.A}.
\begin{theorem} \label{thm.main_theorem}
Under Assumptions \ref{assumption.lipschitz}, \ref{assumption.lowerbound}, and \ref{assumption.unbiased}, when the step size $\eta$ satisfies the following condition:
\begin{align}
1 - \eta L - \eta^2 L^2_{max} Q^2 \geq 0
\end{align}
then, for $T>0$, the expected squared norm of the gradient of $\Lc$ averaged over all $T$ iterations satisfies the following bound: 
\begin{align}
\ex \left[\frac{1}{T} \sum\limits_{t=0}^{T-1}\lVert \nabla \Lc(\gm^t) \rVert^2 \right ] 
&\leq \frac{2\left(\Lc(\gm^0) - \Lc_{inf}\right)}{\eta T}   + \frac{\eta L N \sigma_{max}^2}{K} 
\nonumber \\
&~~~~+ L_{max}^2 \eta^2\sigma_{max}^2 Q^2 N^2 \label{eqn.final_convergence_rate}
\end{align}
\end{theorem}

We note that the bound in Theorem~\ref{thm.main_theorem} converges to  a non-zero value as $T\rightarrow \infty$. The convergence error results from the parallel updates on the coordinate blocks (on $N$) 
, staleness due to multiple local iterations (on $Q$) and due to parallel updates based on horizontal partitioning (on $K$) as well.
With an increase in the number of vertical partitions, the error term increases quadratically. The error also depends quadratically on $Q$, however, in practice, if $Q$ is offset by a suitable learning rate $\eta$, then we can leverage multiple local iterations to achieve faster convergence as we will show in Sec.~\ref{sec.expresults}. However, choosing a very small $\eta$ will decrease the convergence error, but it will but increase the first term on the right hand side of (\ref{eqn.final_convergence_rate}), leading to slower convergence.

%%%%%%%%%%%%%%%%%%%%%%%%%%%%%%%%%%%%%%%%%%%%%%%%%%%%%%%%%%%%%%%%%%%%%%%%%%%%%%%%
%%%%%%%%%%%%%%%%%%%%%%%%%%  EXPERIMENTS   %%%%%%%%%%%%%%%%%%%%%%%%%%%%%%%%%%%%%%
%%%%%%%%%%%%%%%%%%%%%%%%%%%%%%%%%%%%%%%%%%%%%%%%%%%%%%%%%%%%%%%%%%%%%%%%%%%%%%%%

\section{Experimental Results}
\label{sec.expresults}

We verify the convergence properties of TDCD with respect to the different algorithm parameters of the system via a simulation. 
In our experiments, each client has the same number of samples = $\frac{M}{K}$.

\subsection{Datasets} 
We first briefly discuss the two datasets used in this study.

\textbf{Superconductivity (Convex Objective: Ridge Regression):} For the first experiments, we use the Superconductivity dataset~\cite{hamidieh2018data}, which consists of numerical values in all coordinates. The goal is to predict the critical temperature of superconducting materials. %As the coordinates had widely variable range, 
We standardized the dataset before using it by normalizing each coordinate to have zero mean and unit variance. We use $20,000$ samples from the original dataset for training. We use all 81 coordinates and include add one for bias. 

\textbf{MNIST (Non-Convex Objective: CNN):}
We train a CNN model on the MNIST dataset~\cite{bottou1994comparison}. 
MNIST is a set of $28\times28$ pixels hand-written digits images with 
$60$,$000$ digits in the training set and $10$,$000$ digits in the test set.
We use $N=2$ for all the experiments and divide each MNIST image vertically into two parts ($28\times14$). Each client trains a local CNN model with a shared linear classifier layer at the top that uses cross-entropy loss. The local CNNs have two \textit{conv} layers followed by a $256$ dimension embedding layer which is fed into the final classifier layer. The two feature representations of $\BB{R}^{256}$ are inputs to the classifier layer with $\BB{R}^{512}$ input and $\BB{R}^{10}$ output. We thus train the weights of the final layer via TDCD while also updating the local CNNs in each iteration. 

\subsection{Results}
In all figures $N$ represents the number of silos (vertical partitions of the dataset), and $K$ represents the number of clients in each silo.
In each of the experiments, The training loss is calculated using the global model $\gm$ and the full training data matrix every $Q$ iterations. We call every $Q$th iteration a \emph{communication round} because it is when communication between clients and hubs occur.

We first study the performance of TDCD on the convex case of ridge regression in Fig.~\ref{fig.convex}. We start with the impact of varying the number of local iterations $Q$ on the convergence rate.
We  fix the network configuration to $N$=4 silos and $K$=5 clients per silo, with a minibatch size of $B=100$ and learning rate $\eta=0.001$. The results are shown in Fig.~\ref{fig:ridge_Q_diverge}. We observe that with increasing values of $Q$, the convergence rate improves. 
This is intuitive as the clients can train more with a larger number of local rounds between communications, however, as stated in Theorem~\ref{thm.main_theorem}, this can result in a larger convergence error.
This implies that by increasing 
the number of local iterations at clients, we can improve the overall communication efficiency by reducing the total number of communication rounds required for a given loss. 

In Fig.~\ref{fig:ridge_N}, we show the impact of varying the number of vertical partitions on the convergence rate. To observe results at higher granularity, we use a subset of 2000 samples from the original training dataset. We fix $K$=2, $Q$=4, and $B$=20 for this experiment. We observe that the effect of increasing $N$ is observable but not very strong. The inset figure shows the last five communication rounds, and we observe that the convergence rate improves with lower value of $N$, which is as per Theorem~\ref{thm.main_theorem}. 

We next study how the number of workers in a silo effects the convergence rate. The results are shown in Fig.~\ref{fig:ridge_K}. We fix $N$=4, and $Q$=4 and $B=500$. Further, we use the same 2000 data points as in the previous experiment. The inset figure here also shows the last five communication rounds of training. We observe that variation of convergence rate is low with varying $K$. This shows that $K$ does not play a large role as $Q$ in its effect on the convergence rate or convergence error. 

Finally, we study the performance of TDCD with the non-convex objective.
We fix the number of silos at $N$=2 and the learning rate $\eta=0.001$ for all experiments. 
We first investigate the impact of $Q$ on the convergence rate and error. The results are shown in Fig.~\ref{fig:comm_vs_Q}. Here, $K$=10 and $B=640$. 
%The Y-axis represents the training loss. 
We observe that the convergence rate improves radically for larger values of $Q$. 
This result is similar to what we obtained from the convex case. Hence, by choosing $Q$ carefully it is possible to significantly decrease the communication cost without losing performance. 
Lastly, in Fig.~\ref{fig:mnist_k_bsconst}, we explore the effect of varying the number of clients at each silo. We fix the product of $K$ and $B$  to 1250 across the experiments, so that each silo effectively trains on the same number of samples in each experiment. Similar to the convex case, we again observe that the effect of $K$ is very mild. 
Overall, we we observe that TDCD performs well with both convex and non-convex objectives.

%%%%%%%%%%%%%%%%%%%%%%%%%%%%%%%%%%%%%%%%%%%%%%%%%%%%%%%%%%%%%%%%%%%%%%%%%%%%%%%%
%%%%%%%%%%%%%%%%%%%%%%%%%%  CONCLUSION    %%%%%%%%%%%%%%%%%%%%%%%%%%%%%%%%%%%%%%
%%%%%%%%%%%%%%%%%%%%%%%%%%%%%%%%%%%%%%%%%%%%%%%%%%%%%%%%%%%%%%%%%%%%%%%%%%%%%%%%
\section{CONCLUSION}
\label{sec:conclusion}
We have introduced TDCD, a communication efficient decentralized algorithm for a multi tier network model with both horizontally and vertically partitioned data. We provided theoretical analysis of the algorithm convergence and its dependence on the number of vertical partitions, the number of clients in each hub, and the number of local iterations. Finally, we presented experimental results to show convergence of our algorithm in practice. In future work, we plan to explore the possibility of hubs communicating with each other asynchronously to share information.

%%%%%%%%%%%%%%%%%%%%%%%%%%%%%%%%%%%%%%%%%%%%%%%%%%%%%%%%%%%%%%%%%%%%%%%%%%%%%%%%
%%%%%%%%%%%%%%%%%%%%%%%%%%  REFERENCES    %%%%%%%%%%%%%%%%%%%%%%%%%%%%%%%%%%%%%%
%%%%%%%%%%%%%%%%%%%%%%%%%%%%%%%%%%%%%%%%%%%%%%%%%%%%%%%%%%%%%%%%%%%%%%%%%%%%%%%%
%\section{REFERENCES}
%\label{sec:refs}
%
%List and number all bibliographical references at the end of the
%paper. The references can be numbered in alphabetic order or in
%order of appearance in the document. When referring to them in
%the text, type the corresponding reference number in square
%brackets as shown at the end of this sentence \cite{C2}. An
%additional final page (the fifth page, in most cases) is
%allowed, but must contain only references to the prior
%literature.

\bibliographystyle{IEEEbib}
\bibliography{ref}
\clearpage
\appendix
\onecolumn

\title{Supplementary: Multi-Tier Federated Learning for Vertically Partitioned Data}
\date{}

%\begin{document}
\maketitle

\section{Proof of the theorem and supporting lemmas} \label{appendix.A}
In this section provide the proofs of our theorem for convergence and the associated helping lemmas. We are omitting the details about how the data is distributed in the clients. It is same as in the main paper. 

We reiterate the objective function of the tiered decentralized coordinate descent approach with periodic averaging. The objective is to train a global model $\gm$, which is a $d$-vector that can be decomposed as
\[
\gm = [ \gm_{(1)}^T, \ldots, \gm_{(N)}^T]^T
\]
where each $\gmj$ is the block of features, or coordinates, for silo $j$.
The goal of the training algorithm is to minimize an objective function with following structure
\begin{align*}
\Lc (\gm,\BF{\BF{X;y}}) \eqdef  \frac{1}{M} \sum\limits_{i=1}^M f(\gm_{(1)}, \ldots, \gm_{(N)}; \BF{X}^{(i)},\BF{y}^{(i)} ) + \lambda \sum_{d=1}^N \omega(\gm_{(d)})
\end{align*}
where $f$ has the partially separable form
\begin{align*}
f(\gm, \BF{X}^{(i)}; \BF{y}^{(i)}) = f \left( \sum\limits_{d=1}^N \BF{X}^{(i)}_{(d)} \gm_{(d)}\;, \BF{y}^{(i)}\right).
\end{align*}
The functions $\omega(\cdot)$ constitute a regularizer, and $\lambda$ is a hyperparameter.  
A concrete example of the loss function is an $L_2$ regularized square loss function for empirical risk minimization:
\begin{align*}
\Lc (\gm, \BF{\BF{X;y}}) &= \frac{1}{2M}||\BF{X}\gm -\BF{y}\|_2^2 + \frac{\| \gm \|_2^2}{2}. 
\end{align*}

\subsection{NOTATION} \label{sec.notation}
We first define the notations to be used in analyzing the convergence of TDCD.

\begin{itemize}
    \item $\gm$ the $D\times 1$ global model. 
    \item $\gmj$ is the $j^{th}$ block of $\gm$, so that $\gm = [ \gm_{(1)}^T, \ldots, \gm_{(N)}^T]^T$. Note that $\gmj$ is a virtual vector. It is realized at a hub $j$ every $Q$ iterations, but we will study the evolution of this virtual vector in every iteration.
    \item $\tjk^t \in \BB{R}^{D_j}$ are the local versions of the coordinates of the weight vector $\gmj^t$ that each client $k$ if hub $j$ updates. 
    \item $\bm{\gm}_{-j}^t$ is the vector of all coordinates in $\gm$, excluding block $j$, at iteration $t$.
    \item $\tjk$ is the local copy of $\gmj$ at client $k$ in silo $j$, so that
        $\gmj = \frac{1}{K} \sum_{k=1}^K \tjk$.
    \item $\ykj^t = [\tjminus, \tjk^t]$ is the parameter vector that client $j$ in silo $k$ 
     at iteration $t$, where $t_0$ is the iteration that client $k$ last updated $\tjminus$.
    \item $\gkj( \ykj; \zb)$ is the partial derivative of $\Lc$ with respect to coordinate block $j$,
    computed at client $k$ in silo $j$ using the coordinates and rows at client $k$ corresponding to minibatch $\zeta$. For simplicity, we will write $\gkj(\ykj)$ when it is clear which minibatch is used.
    \item $\bm{G}^t = [(\bm{G}_{(1)}^t)^T, \ldots, (\bm{G}_{(N)}^t)^T]^T$, where $\Gj^t =  \frac{1}{K} \sum_{k=1}^K \gkj(\ykj)$
    \item $\bm{H}^t = [(\bm{H}_{(1)}^t)^T, \ldots, (\bm{H}_{(N)}^t)^T]^T$, where $\Hj^t = \frac{1}{K} \sum_{k=1}^K \pdj \Lc (\ykj)$
    \end{itemize}
Further, for any vector $\bm{m} \in \mathbb{R}^d$, $\bm{m}_{(l)}$ denotes the $l$th block corresponding to the $l$th silo in vector $\bm{m}$.

It to be noted that components on $\gm$ i.e. $\gmj$ are realized every $Q$ iterations when the hubs communicate with clients and with other hubs, but we will study the evolution of these virtual vectors at each iteration. 
Therefore, based on the above definitions, assumptions and the TDCD algorithm, we can express the evolution of the virtual global parameter/weight vector in the following form:
\begin{align}
\gm = 
\begin{bmatrix}
\tilde{\bm{\theta}}_{(1)}^t\\
\tilde{\bm{\theta}}_{(2)}^t\\
\vdots \\
\tilde{\bm{\theta}}_{(N)}^t
\end{bmatrix}_{D\times1}
=
\frac{1}{K} \begin{bmatrix}
\sum_{k=1}^K \theta_{k,1}^t\\
\sum_{k=1}^K \theta_{k,2}^t\\
\vdots \\
\sum_{k=1}^K \theta_{k,N}^t
\end{bmatrix}_{D\times1}
\quad ,
\gm^{t+1} = \gm^t -
\frac{\eta}{K} \begin{bmatrix}
\sum_{k=1}^K g_{k,1}(y_{k,1}^t; \zeta^t)\\
\sum_{k=1}^K g_{k,2}(y_{k,2}^t; \zeta^t)\\
\vdots \\
\sum_{k=1}^K g_{k,N}(y_{k,N}^t; \zeta^t)
\end{bmatrix}_{D\times1}\label{eqn.mainequation}
\end{align}

In this case, we update all coordinates of the global weight vector $\gm^t$ , virtually at each time step $t$.
we have the virtual gradient at each time instant $t$ as: 
\[
\BF{G^t} \overset{\bigtriangleup}{=}
\frac{1}{K}
\begin{bmatrix}
\sum_{k=1}^K g_{k,1}(y_{k,1}^t; \zeta^t)\\
\sum_{k=1}^K g_{k,2}(y_{k,2}^t; \zeta^t)\\
\vdots \\
\sum_{k=1}^K g_{k,N}(y_{k,N}^t; \zeta^t)
\end{bmatrix}
\]

\subsection{ASSUMPTIONS}
We make the following assumptions about the loss function $\Lc$ and the gradients $\gkj$ at each client.
\begin{assumption} \label{assumption.lipschitz} 
    The gradient of the loss function is Lipschitz continuous with constant $L$; further, the partial derivative of $\BC{L}$ with respect to each coordinate block $j$ is Lipschitz continuous with constant $L_j$, i.e.,  for all $\bm{\theta}_1, \bm{\theta}_2 \in \mathbb{R}^D$   
   \begin{align}
     \lVert \nabla \Lc(\bm{\theta}_1) - \nabla \Lc(\bm{\theta}_2) \parallel &\leq L\parallel \bm{\theta}_1 - \bm{\theta}_2\rVert \\         
    \lVert \pdj \Lc(\bm{\theta}_1) - \pdj \Lc(\bm{\theta}_2) \parallel &\leq L_j\parallel \bm{\theta}_1 - \bm{\theta}_2\rVert.     
    \end{align}
   \end{assumption}

\begin{assumption}\label{assumption.lowerbound}
The function $\Lc$ is lower bounded so that for all $\bm{\theta} \in \mathbb{R}^D, \Lc(\bm{\theta}) \geq \Lc_{inf}$.
\end{assumption}

\begin{assumption}\label{assumption.unbiased}
Let $\zeta$ be a mini-batch drawn uniformly at random from all samples. 
We assume that the data is distributed so that, for all $\bm{\theta} \in \mathbb{R}^D$
\begin{align}
    &\ex_{\zeta | \bm{\theta}}\left[\gkj(\bm{\theta}) \right]= \pdj \Lc(\bm{\theta})\\
      &\ex_{\zeta | \bm{\theta}} \left[\| \gkj(\bm{\theta}) - \pdj \Lc(\bm{\theta}) \|^2 \right] \leq \sigma^2_j.
\end{align}
\end{assumption}

We also use the following definitions: 
\[
L_{max} = \max_{1\leq j \leq N} L_j ~,~ \sigma_{max} = \max_{1 \leq j \leq N} \sigma_j
\]

%%%%%%%%%%%%%%%%%%%%%%CONVERGENCE ANALYSIS%%%%%%%%%%%%%%%%%%%%%%%%%%%
\subsection{CONVERGENCE ANALYSIS}
We can write the evolution of the global model from Sec.~\ref{sec.notation} as:
\begin{align}
\gm^{t+1} = \gm^t - \eta \bm{G}^t. \label{eqn.original}
\end{align}
We will study the evolution of this global model.
We will use $\ext$ to denote $\ex_{\zeta^t | \gm^t}$.

\begin{align}
    \ext[\Lc(\gm^{t+1})] - \Lc(\gm^t)
    &\leq  \ext \langle \nabla \Lc(\gm^t), \gm^{t+1} -\gm^t \rangle + \frac{L}{2}~  \ext~ \| \gm^{t+1} -\gm^t \|^2 
    \\ 
    &\leq 
        - \underbrace{\eta ~ \ext\langle \nabla \Lc(\gm^t) , \bm{G}^t \rangle}_{T_1}  
        + \underbrace{\frac{\eta^2 L}{2} ~  \ext ~\| \bm{G}^t \|^2}_{T_2}  \label{t1.eq}
\end{align}

We will use the following lemma to simplify $T_1$.
\begin{lemma} \label{lemma.variance}
\[
\ext\lVert \Gj^t - \Hj^t \rVert^2 \leq \frac{\sigma_{max}^2}{K}
\]
\end{lemma}

\begin{proof}
\begin{align}
&\ext \lVert \Gj^t - \Hj^t \rVert^2
\\
&= \ext \left\| \frac{1}{K} \sum\limits_{k=1}^K \gkj(\ykj^t) -  \frac{1}{K} \sum\limits_{k=1}^K \pdj\Lc (\ykj^t)  \right\|^2
\\
&= \frac{1}{K^2} \ext \left [ \sum\limits_{k=1}^K  \lVert \gkj(\ykj^t)  - \pdj \Lc(\ykj^t)  \rVert^2 
        + \sum\limits_{k=1}^K \sum\limits_{l=1, l\neq k}^K \langle \gkj(\ykj^t) 
        - \pdj\Lc(\ykj^t), g_{l,j}(\ykj)^t) - \pdj \Lc(\bm{y}_{l,j}^t) \rangle \right]
\\
&{=} \frac{1}{K^2}  \sum\limits_{k=1}^K \ext  \lVert \gkj(\ykj)^t) - \pdj \Lc(\ykj^t)  \rVert^2 
\\  
    &~~~+ \frac{1}{K^2} \sum\limits_{k=1}^K \sum\limits_{l=1, l\neq k}^K 
    \langle \ext\left[\gkj(\ykj^t) - \nabla_{(j)} \Lc(\ykj^t)\right], \ext\left[g_{l,j}(\bm{y}_{l,j}^t) - \nabla_{(j)} \Lc (\bm{y}_{l,j}^t)\right] \rangle \label{eqn.cross_term_zero}
\end{align}
 Applying Assumption~\ref{assumption.unbiased} to (\ref{eqn.cross_term_zero}), we observe that  we can bound the variance in the first sum and that the cross terms in the double summation evaluate to zero. We therefore have the following:
\begin{align}
\ex_{\zeta^t | \gm^t} \lVert \Gj^t - \Hj^t \rVert^2 
&\leq  \frac{1}{K^2} \sum\limits_{k=1}^K  \sigma_{j}^2 \\
&\leq  \frac{\sigma_{j}^2}{K} \\
&\leq  \frac{\sigma_{max}^2}{K}.
\end{align}
\end{proof}

\begin{lemma} \label{lemma.two}
\begin{align}
\ext \parallel\bm{G}^t\parallel^2 
\leq \frac{N \sigma_{max}^2}{K}  
    + \frac{1}{K} \sum\limits_{j=1}^N \sum\limits_{k=1}^K \lVert  \nabla_{(j)} \Lc(\ykj^t) \rVert^2
\end{align}
\end{lemma}

\begin{proof}
\begin{align}
\ext \parallel\bm{G}^t\parallel^2 &= \sum\limits_{j=1}^N \ext \lVert \Gj^t \rVert^2 
    \\
    &\overset{(a)}{=} \sum\limits_{j=1}^N \ext \left [ \lVert \Gj^t - \ext [\Gj^t] \rVert^2 \right ]  +  \sum\limits_{j=1}^N \lVert \ext[\Gj^t] \rVert^2
    \\
    &= \sum\limits_{j=1}^N \ext \left [ \lVert \Gj^t - \Hj^t \rVert^2\right ]   +  \sum\limits_{j=1}^N \lVert \Hj^t \rVert^2
    \\
    &\overset{(b)}{\leq}\sum\limits_{j=1}^N \frac{\sigma_{max}^2}{2K} 
    + \sum\limits_{j=1}^N \lVert \frac{1}{K}\sum\limits_{k=1}^K \nabla_{(j)} \Lc (\ykj^t) \rVert^2 
    \\
    &\overset{(c)}{\leq} \frac{N \sigma_{max}^2}{K}  
    + \frac{1}{K} \sum\limits_{j=1}^N \sum\limits_{k=1}^K \lVert  \nabla_{(j)} \Lc(\ykj^t) \rVert^2 
     \label{eqn.mainequationT2}
\end{align}
where $(a)$ follows directly from Assumption~\ref{assumption.unbiased} and the definitions of $\bm{G}^t$ and $\bm{H}^t$.  The simplification in $(b)$ is from Lemma~\ref{lemma.variance}, and $(c)$ is because $\sum_{i=1}^N \parallel a_i \parallel^2 \leq \parallel \sum_{i=1}^N a_i \parallel^2 \leq N \sum_{i=1}^N \parallel a_i \parallel^2 $ .
\end{proof}

We next present a lemmas to lower bound $T_1$.
\begin{lemma}\label{lemma.one}
\begin{align}
\ext \left[ \langle \nabla \Lc(\gm^t) , \bm{G}^t \rangle \right]= 
\frac{1}{2} \lVert \nabla \Lc(\gm^t) \rVert^2 + \frac{1}{2K} \sum\limits_{j=1}^N \sum\limits_{k=1}^K \lVert  \nabla_{(j)} \Lc(\ykj^t )\rVert^2  - \frac{1}{2K} \sum\limits_{j=1}^N \sum\limits_{k=1}^K \lVert \nabla_{(j)} \Lc(\gm^t)- \nabla_{(j)} \Lc(\ykj^t) \rVert^2
\end{align}
\end{lemma}

\begin{proof}
We have:
\begin{align}
    & \ext \langle \nabla \Lc(\gm^t) , \bm{G}^t \rangle \label{eqn.mainequationT1} 
    \\
    &= \sum\limits_{j=1}^N \left \langle \nabla_{(j)} \Lc(\gm^t),\ext\left[\Gj^t\right]  \right \rangle 
    \\
    &\overset{(a)}{=} \sum\limits_{j=1}^N \left \langle \nabla_{(j)} \Lc(\gm^t), \bm{H}_{(j)}^t  \right \rangle 
    \\
    &= \frac{1}{K} \sum\limits_{j=1}^N \sum\limits_{k=1}^K \langle \nabla_{(j)} \Lc(\gm^t), \nabla_{(j)} \Lc(\ykj^t) \rangle 
    \\
    &= \frac{1}{2K} \sum\limits_{j=1}^N \sum\limits_{k=1}^K \left [ \lVert \nabla_{(j)} \Lc(\gm^t) \rVert^2 +  \lVert  \nabla_{(j)} \Lc(\ykj^t) \rVert^2 
    - \lVert \nabla_{(j)} \Lc(\gm^t)- \nabla_{(j)} \Lc(\ykj^t) \rVert^2 \right ]     
    \\
    &= \frac{1}{2} \lVert \nabla \Lc(\gm^t) \rVert^2 + \frac{1}{2K} \sum\limits_{j=1}^N \sum\limits_{k=1}^K \lVert  \nabla_{(j)} \Lc(\ykj^t) \rVert^2 
    - \frac{1}{2K} \sum\limits_{j=1}^N \sum\limits_{k=1}^K \lVert \nabla_{(j)} \Lc(\gm^t)- \nabla_{(j)} \Lc(\ykj^t) \rVert^2
    \label{eqn.mainequationT1end}
\end{align}
where, in $(a)$, we use Assumption~\ref{assumption.unbiased}.
\end{proof}

We now define a lemma to relate the expected local and expected global gradient.

\begin{lemma} \label{lemma.expected_local_global_gradient}
    \[
    \ext \| \gkj(\bm{\theta})\| ^2 \leq \sigma_j^2 
        + \lVert \nabla_{(j)} \Lc(\bm{\theta}) \rVert^2 
    \]
\end{lemma}
\begin{proof}
We observe that
\highlight{
\begin{align}
\ext \| \gkj(\bm{\theta}) \|^2 &= \ext \| \gkj(\bm{\theta}) - \ext[\gkj(\bm{\theta})] \|^2 + \| \ext[\gkj(\bm{\theta})]\|^2 \\
&=\ext \|\gkj(\bm{\theta})  - \pdj \Lc(\bm{\theta}) \|^2 + \| \pdj \Lc(\bm{\theta}) \|^2\\
&\leq \sigma^2_j + \| \pdj \Lc(\bm{\theta}) \|^2
\end{align}
}
where the last inequality follows from Assumption~\ref{assumption.unbiased}.
\end{proof}

\subsection{PROOF OF THEOREM~\ref{thm.main_theorem}}

We now prove our main result. We return to the expression in (\ref{t1.eq}),

\begin{align}
    \ext[\Lc(\gm^{t+1})] - \Lc(\gm^t)
    &\leq 
        - \eta ~ \ext\langle \nabla \Lc(\gm^t) , \bm{G}^t \rangle 
        + \frac{\eta^2 L}{2} ~  \ext ~\| \bm{G}^t \|^2 
\end{align}
From Lemmas~\ref{lemma.two} and \ref{lemma.one}, we now have
\begin{align}
&\ext[\Lc(\gm^{t+1})] - \Lc(\gm^t) 
\leq \\
&- \frac{\eta}{2} \lVert \nabla \Lc(\gm^t) \rVert^2 - \frac{\eta}{2K} \sum\limits_{j=1}^N \sum\limits_{k=1}^K \lVert  \nabla_{(j)} \Lc(\ykj^t )\rVert^2 
 + \frac{\eta}{2K} \sum\limits_{j=1}^N \sum\limits_{k=1}^K \lVert \nabla_{(j)} \Lc(\gm^t)- \nabla_{(j)} \Lc(\ykj^t) \rVert^2 \\
    &+ \frac{\eta^2 L}{2K} \sum\limits_{j=1}^N \sum\limits_{k=1}^K \lVert  \nabla_{(j)} \Lc(\ykj^t) \rVert^2 + \frac{\eta^2 L N \sigma_{max}^2}{2K}  \\
\leq& - \frac{\eta}{2} \lVert \nabla \Lc(\gm^t) \rVert^2 - \frac{\eta}{2K}(1 - \eta L) \sum\limits_{j=1}^N \sum\limits_{k=1}^K \lVert  \nabla_{(j)} \Lc(\ykj^t )\rVert^2
 + \frac{\eta L_{max}^2}{2K} \sum\limits_{j=1}^N \sum_{k=1}^K 
    \lVert  \gm^t -  \ykj^t \rVert^2 \label{t2.eq} 
    \\
    & + \frac{\eta^2 L N \sigma_{max}^2}{2K} 
\end{align}
where (\ref{t2.eq}) follows from Assumption~\ref{assumption.lipschitz}.
Rearranging, we get
\begin{align}
 \lVert \nabla \Lc(\gm^t) \rVert^2 \leq& \frac{2}{\eta}\left(\Lc(\gm^t) - \ext[\Lc(\gm^{t+1})]\right)  + \frac{ L_{max}^2}{K} \sum\limits_{j=1}^N \sum_{k=1}^K 
    \lVert  \gm^t -  \ykj^t \rVert^2 + \frac{\eta L N \sigma_{max}^2}{K} \\
    &- \frac{1}{K}(1 - \eta L) \sum\limits_{j=1}^N \sum\limits_{k=1}^K \lVert  \nabla_{(j)} \Lc(\ykj^t )\rVert^2
\end{align}

We now take the total expectation and average all iterates from $t=0, \ldots, T$ 
\begin{align}
\ex\left[ \frac{1}{T} \sum_{t=0}^{T-1} \lVert \nabla \Lc(\gm^t) \rVert^2 \right] \leq& \frac{2}{\eta T}\left(\Lc(\gm^0) - \Lc_{inf}\right)  + \frac{ L_{max}^2}{TK} \sum_{t=0}^{T-1} \sum\limits_{j=1}^N \sum_{k=1}^K 
   \ex \lVert  \gm^t -  \ykj^t \rVert^2 + \frac{\eta L N \sigma_{max}^2}{K} \\
    &- \frac{1}{TK}(1 - \eta L) \sum_{t=0}^{T-1} \sum\limits_{j=1}^N \sum\limits_{k=1}^K \ex \lVert  \nabla_{(j)} \Lc(\ykj^t )\rVert^2 \label{b1.eq}
\end{align}

We use the following lemmas to simplify the expression above.

\begin{lemma} \label{lemma.intermediate_helping}
Let $t_0$ be the most recent iteration in which the hubs exchanged information and sent new models to the clients prior to iteration $t$. Then
\begin{align}
    \frac{1}{K}  \sum_{k=1}^K  \|  \gm^t_{(j)} -  \bm{y}_{k,j, (j)}^t \|^2 
    \leq \frac{Q\eta^2}{K}  \sum_{k=1}^K \sum\limits_{\tau=t_0}^{t-1}  \| \gkj(\ykj^{\tau}) \|^2 
\end{align} 
\end{lemma}
\begin{proof}
\begin{align}
 \frac{1}{K}  \sum_{k=1}^K \|  \gm^t_{(j)} -  \bm{y}_{k,j, (j)}^t \|^2 &= \frac{1}{K}  \sum_{k=1}^K \parallel \gmj^t - \bm{\theta}_{k,j}^t \parallel^2 \\
 &=\frac{1}{K}  \sum_{k=1}^K \parallel  \bm{\theta}_{k,j}^t - \gmj^{t_0} - (\gmj^t - \gmj^{t_0}) \parallel^2 \\
 &= \left\| \left(\bm{\theta}_{k,j}^t - \gmj^{t_0}\right) - \left(\frac{1}{K} \sum_{l=1}^K (\bm{\theta}_{l,j}^t - \gmj^{t_0}) \right) \right\|^2
\end{align}

We observe that, for an arbitrary set of vectors $\{\bm{z}_l\}_{k=1}^K$,
\begin{align}
\frac{1}{K}  \sum_{k=1}^K \left\| \bm{z}_k - \frac{1}{K} \sum_{l=1}^K \bm{z}_l \right\|^2 \leq \frac{1}{K}  \sum_{k=1}^K \| \bm{z}_k\|^2.
\end{align}
Letting $\bm{z}_l = \bm{\theta}_{l,j}^t - \gmj^{t_0}$, 
we have
\begin{align}
 \frac{1}{K}  \sum_{k=1}^K \|  \gm^t_{(j)} -  \bm{y}_{k,j, (j)}^t \|^2 &\leq 
 \frac{1}{K}  \sum_{k=1}^K \left\| \bm{\theta}_{k,j}^t - \gmj^{t_0}\right\|^2 \\
    &\leq \frac{1}{K}  \sum_{k=1}^K (t-1-t_0) \eta^2 \sum\limits_{\tau=t_0}^{t-1} \parallel  \gkj(\ykj^{\tau}) \parallel^2 \\
    &\leq \frac{Q\eta^2}{K}  \sum_{k=1}^K \sum\limits_{\tau=t_0}^{t-1} \parallel \gkj(\ykj^{\tau}) \parallel^2.
\end{align}
\end{proof}

\begin{lemma}\label{lemma.networkerror}
\begin{align}
&  \frac{L_{max}^2}{K} \sum\limits_{j=1}^N \sum_{k=1}^K 
    \lVert  \gm^t -  \ykj^t \rVert^2 
 \leq \frac{\eta^2 L_{max}^2 Q N}{K} 
    \left [ \sum\limits_{j=1}^N  \sum\limits_{\tau = t_0}^{t-1}\sum\limits_{p=1}^{K} \parallel g_{p,j}(\bm{y}_{p,j}^{\tau}) \parallel^2 \right] 
\end{align}
\begin{proof}
\begin{align}
    &  \frac{L_{max}^2}{K} \sum\limits_{j=1}^N \sum_{k=1}^K 
    \lVert  \gm^t -  \ykj^t \rVert^2  \label{explanation.0} 
    \\
    & = \frac{L_{max}^2}{K} \sum\limits_{j=1}^N \sum_{k=1}^K
    \left[ \sum\limits_{l \neq j}^N  \lVert  \gm^t_{(l)} - \gm^{t_0}_{(l)} \rVert^2 + \lVert  \gm^t_{(j)} -  \bm{y}_{k,j, (j)}^t \rVert^2 \right ] \label{explanation.1} 
    \\
    & = \frac{L_{max}^2}{K} \sum\limits_{j=1}^N \sum_{k=1}^K
    \left [ \sum\limits_{l \neq j}^N  \lVert  \gm^{t_0}_{(l)} - \frac{\eta}{K} \sum\limits_{\tau = t_0}^{t-1}\sum\limits_{p=1}^{K} g_{p,l}(\bm{y}_{p,l}^{\tau}) 
     - \gm^{t_0}_{(l)} \rVert^2 + \lVert  \gm^t_{(j)} -  \bm{y}_{k,j, (j)}^t \rVert^2 \right ]
\\
    &{\leq} \frac{L_{max}^2}{K} \sum\limits_{j=1}^N \sum_{k=1}^K  \frac{\eta^2}{K} (t-1-t_0) \sum\limits_{l \neq j}^N  \sum\limits_{\tau = t_0}^{t-1}\sum\limits_{p=1}^{K} \lVert  g_{p,l}(\bm{y}_{p,l}^{\tau}) \rVert^2
     + \frac{L_{max}^2}{K} \sum\limits_{j=1}^N \sum_{k=1}^K \lVert  \gm^t_{(j)} -  \bm{y}_{k,j, (j)}^t \rVert^2 \label{equation.c}
       \\
    &{\leq} \frac{\eta^2 L_{max}^2 Q }{K^2} \sum\limits_{j=1}^N \sum_{k=1}^K \sum\limits_{l \neq j}^N  \sum\limits_{\tau = t_0}^{t-1}\sum\limits_{p=1}^{K} \lVert  g_{p,l}(\bm{y}_{p,l}^{\tau}) \rVert^2
     + \frac{L_{max}^2}{K} \sum\limits_{j=1}^N \sum_{k=1}^K \lVert  \gm^t_{(j)} -  \bm{y}_{k,j, (j)}^t \rVert^2 \label{equation.d}
     \\
    &= \frac{\eta^2 L_{max}^2 Q (N-1) }{K} \sum\limits_{j=1}^N  \sum\limits_{\tau = t_0}^{t-1}\sum\limits_{p=1}^{K} \lVert  g_{p,j}(\bm{y}_{p,j}^{\tau}) \rVert^2
     + \frac{L_{max}^2}{K} \sum\limits_{j=1}^N \sum_{k=1}^K \lVert  \gm^t_{(j)} -  \bm{y}_{k,j, (j)}^t \rVert^2   
     \\
     &{\leq} \frac{\eta^2 L_{max}^2 Q (N-1) }{K} \sum\limits_{j=1}^N  \sum\limits_{\tau = t_0}^{t-1}\sum\limits_{p=1}^{K} \lVert  g_{p,j}(\bm{y}_{p,j}^{\tau}) \rVert^2
     + L_{max}^2 \sum\limits_{j=1}^N \frac{Q\eta^2}{K}  \sum_{k=1}^K \sum\limits_{\tau=t_0}^{t-1} \parallel  \gkj(\ykj^{\tau}) \parallel^2 \label{equation.e}
\\
&= \frac{\eta^2 L_{max}^2 Q N }{K} \sum\limits_{j=1}^N  \sum\limits_{\tau = t_0}^{t-1}\sum\limits_{p=1}^{K} \lVert  g_{p,j}(\bm{y}_{p,j}^{\tau}) \rVert^2. \label{eqn.network_error_csgd}
\end{align}
Here, in (\ref{explanation.1}) we use the fact that each hub sent the updated model to its clients in iteration $t_0$, where $t-t_0 \leq Q$.  In (\ref{equation.c}), we use the inequality $\parallel \sum_{i=1}^N a_i \parallel^2 \leq N \sum_{a=1}^N \parallel a_i \parallel^2 $ and in (\ref{equation.d}), we use the fact that $Q\geq t-1-t_0$.
Finally, in (\ref{equation.e}) we use Lemma~\ref{lemma.intermediate_helping}.
\end{proof}
\end{lemma}

Applying Lemma~\ref{lemma.networkerror}, we can further bound (\ref{b1.eq}) as
\begin{align}
\ex\left[ \frac{1}{T} \sum_{t=0}^{T-1} \lVert \nabla \Lc(\gm^t) \rVert^2 \right] \leq&\frac{2}{\eta T}\left(\Lc(\gm^0) - \Lc_{inf}\right)  +  \frac{\eta^2 L_{max}^2 Q N }{TK} \sum_{t=0}^{T-1} \sum\limits_{j=1}^N  \sum\limits_{\tau = t_0}^{t-1}\sum\limits_{p=1}^{K}\ex \lVert  g_{p,j}(\bm{y}_{p,j}^{\tau}) \rVert^2\\
&+ \frac{\eta L N \sigma_{max}^2}{K} 
    - \frac{1}{TK}(1 - \eta L) \sum_{t=0}^{T-1} \sum\limits_{j=1}^N \sum\limits_{k=1}^K \ex \lVert  \nabla_{(j)} \Lc(\ykj^t )\rVert^2
    \\
\leq &\frac{2}{\eta T}\left(\Lc(\gm^0) - \Lc_{inf}\right)
 + \frac{\eta^2 L_{max}^2 Q^2 N}{TK} \sum\limits_{t=0}^{T-1}\sum\limits_{j=1}^N  \sum\limits_{p=1}^{K} \ex \lVert  g_{p,j}(\bm{y}_{p,j}^{t}) \rVert^2 \label{eqn.doublesummation_over_time} \\
   &   + \frac{\eta L N \sigma_{max}^2}{K}
       - \frac{1}{TK}(1 - \eta L) \sum_{t=0}^{T-1} \sum\limits_{j=1}^N \sum\limits_{k=1}^K \ex \lVert  \nabla_{(j)} \Lc(\ykj^t )\rVert^2 
       \\
       \leq &\frac{2}{\eta T}\left(\Lc(\gm^0) - \Lc_{inf}\right)
 + \frac{\eta^2 L_{max}^2 Q^2 N}{TK} \sum\limits_{t=0}^{T-1}\sum\limits_{j=1}^N  \sum\limits_{p=1}^{K} \left(\sigma_j^2 
        + \highlight{ \ex }\lVert \nabla_{(j)} \Lc(\bm{y}_{p,j}^t) \rVert^2 \right) \\
   &   + \frac{\eta L N \sigma_{max}^2}{K}
       - \frac{1}{TK}(1 - \eta L) \sum_{t=0}^{T-1} \sum\limits_{j=1}^N \sum\limits_{k=1}^K \ex \lVert  \nabla_{(j)} \Lc(\ykj^t )\rVert^2 \label{varbound1.eq} 
       \\
       {\leq}& \frac{2}{\eta T}\left(\Lc(\gm^0) - \Lc_{inf}\right) + \frac{\eta L N \sigma_{max}^2}{K} 
         + L_{max}^2 \eta^2\sigma_{max}^2 Q^2 N^2
\\
&~~~ - \frac{1}{TK} \left[ 1 - \eta L - \eta^2 L^2_{max} Q^2 \right]\sum\limits_{t=0}^{T-1} \sum\limits_{j=1}^N \sum\limits_{k=1}^K \ex \lVert  \nabla_{(j)} \Lc(y_{k,j}^t) \rVert^2
\end{align}
where we can simplify the double summation in (\ref{eqn.doublesummation_over_time}) because $t-1 - t_0 \leq Q$, and in (\ref{varbound1.eq}), we apply Lemma~\ref{lemma.expected_local_global_gradient}.
\linebreak
\linebreak
Assuming $\eta$ is chosen so that $1 - \eta L - \eta^2 L^2_{max} Q^2 \geq 0$,
we have
\begin{align}
\ex \left[\frac{1}{T} \sum\limits_{t=0}^{T-1}\lVert \nabla \Lc(\gm^t) \rVert^2 \right ] \leq \frac{2}{\eta T}\left(\Lc(\gm^0) - \Lc_{inf}\right)   + \frac{\eta L N \sigma_{max}^2}{K} 
        + L_{max}^2 \eta^2\sigma_{max}^2 Q^2 N^2 \label{eqn.final_convergence_rate}
\end{align}
This completes the proof. $\QEDB$

%\vfill
%\end{document}

%\includepdf[pages=-]{icassp_supplementary.pdf}

\end{document}